\documentclass[]{article}

\usepackage{times}

\usepackage{uai20162e}
\usepackage{hyperref}
\usepackage{url}

\usepackage{natbib}
\usepackage{epsfig}
\usepackage{graphicx}
\usepackage{amsmath}
\usepackage{amssymb}
\usepackage{amsthm}
\usepackage{bbold}
\usepackage{flushend}
\usepackage{subfig}
\usepackage{array}
\usepackage{comment}
\usepackage{subfig}

\usepackage[final]{pdfpages}

\usepackage{enumitem}

\usepackage[ruled,vlined,linesnumbered,resetcount]{algorithm2e}

\newtheorem{theorem}{Theorem}[section]
\newtheorem{lemma}[theorem]{Lemma}

\theoremstyle{definition}
\newtheorem{definition}[theorem]{Definition}

\theoremstyle{remark}

\newcommand{\R}{\mathbb{R}}
\newcommand{\angleb}[1]{\langle #1 \rangle}

\newcommand{\at}[1]{\langle #1 \rangle}
\newcommand{\XX}{\textbf{X}}
\newcommand{\YY}{\textbf{Y}}

\newcommand{\YahooLTR}{\textsc{Yahoo!\,LTR}}

\newcommand{\Grain}{\textsc{Agricultural}}

\DeclareMathOperator*{\argmax}{argmax} 

\allowdisplaybreaks

\title{Efficient Feature Group Sequencing for 
Anytime Linear Prediction}

\author{ {\bf Hanzhang Hu} \\
hanzhang@cs.cmu.edu \\
\And
{\bf Alexander Grubb}  \\
agrubb@cs.cmu.edu \\
\And
{\bf J. Andrew Bagnell}   \\
dbagnell@ri.cmu.edu \\
\And
{\bf Martial Hebert} \\
hebert@ri.cmu.edu
}

\begin{document}

\maketitle

\begin{abstract}

We consider \textit{anytime} linear prediction 
in the common machine learning setting, where features are 
in groups that have costs. We achieve anytime (or interruptible)
predictions by sequencing the computation of feature groups and
reporting results using the computed features at interruption. 
We extend Orthogonal Matching Pursuit (OMP) and Forward 
Regression (FR) to learn the sequencing greedily under 
this group setting with costs. We theoretically guarantee 
that our algorithms achieve near-optimal linear predictions 
at each budget when a feature group is chosen. With a novel 
analysis of OMP, we improve its theoretical bound to the same 
strength as that of FR. In addition, we develop a novel 
algorithm that consumes cost $4B$ to approximate the optimal 
performance of \textit{any} cost $B$, and prove that with 
cost less than $4B$, such an approximation is impossible. 
To our knowledge, these are the first anytime bounds 
at \textit{all} budgets. We test our algorithms on two 
real-world data-sets and evaluate them in terms of anytime 
linear prediction performance against cost-weighted Group 
Lasso and alternative greedy algorithms.


\end{abstract}

\section{INTRODUCTION AND BACKGROUND}

First defined by \cite{anytime}, anytime predictors output 
valid results even if they are interrupted at any point in time. The results improve with resources spent. In this work, we propose an anytime linear prediction algorithm under the common 
machine learning setting, where features are computed in groups with associated costs. We further assume that the cost of 
prediction is dominated by feature computation. Hence, 
we can achieve anytime predictions by computing feature groups
in a specific order and outputting linear predictions 
using only computed features at interruption.

Formally, we are given $n$ samples $(x^i, y^i)$ from 
a feature matrix $X \in \mathbb{R}^{n \times D}$ and a response vector $Y \in \mathbb{R}^n$. We also have a partition of the
$D$ feature dimensions into $J$ feature groups, 
$\mathcal{G}_1, \mathcal{G}_2, ..., \mathcal{G}_J$, and 
an associated cost of each group $c(\mathcal{G}_j)$. Our anytime prediction approach learns a sequencing of the
feature groups, $G$ = $g_1, g_2,..., g_J$. 
For each budget limit $B$, the computed
groups at cost $B$ is a prefix of the sequencing, $G_{\angleb{B}} = g_1, g_2,.., g_{J_{\angleb{B}}}$, 
where 
$J_{\angleb{B}} = \max \{ j\leq J | \sum _{i\leq j} c(g_i) \leq B \}$ indexes the last group within the budget $B$. 
An ideal anytime algorithm seeks a sequencing $G$ to minimize risk at all budgets $B$:
\begin{align}
\label{eq:risk}
R(G_{\angleb{B}}) :=   \min _{w}
    \frac{1}{2n} \Vert Y - X_{{G_{\angleb{B}}}} {w} \Vert_2^2 + \frac{\lambda}{2} \Vert w\Vert_2^2,
\end{align}
where $X_{G_{\angleb{B}}}$ contains features in $G_{\angleb{B}}$, $w$ is the associated linear predictor coefficient, and $\lambda$ is a regularizing constant.
Equivalently, if we assume that the $y^i$'s have unit variance and zero mean by normalization, we can maximize the explained variance, 
$    \label{eq:exp-var}
    F(G_{\angleb{B}}) := \frac{1}{2n} Y^TY - R(G_{\angleb{B}}). $

The above optimization problem is closest to the problem of subset selection 
for regression \citep{kemp}, which selects at most $k$ features to optimize a 
linear regression. The problem is also similar to that of sparse model recovery 
\citep{lasso}, which recovers coefficients of a true linear model. 
One common approach to these two problems is to select the features greedily 
via Forward Regression (FR) \citep{FR} or Orthogonal Matching Pursuit (OMP) \citep{omp}. 
Forward Regression greedily selects features 
that maximize the marginal increase in explained 
variance at each step. 
Orthogonal Matching Pursuit selects features as follows. The linear model coefficients of the unselected 
features are set to zero. At each step, the feature whose model 
coefficient has the largest gradient of the risk is selected. 
In this work, we extend FR and OMP to the setting where features are in  
groups that have costs. The extension to FR is intuitive: we
only need to select feature groups using their marginal gain in 
objective per unit cost instead of using just the marginal gain. However, we have
two notes about the extension to OMP. First, to incorporate feature
costs, we need to evaluate a feature based on the squared norm of 
the associated weight vector gradient per unit cost instead of just the gradient norm. 
Second, when we compute the gradient norm for a feature group, $\nabla_g$, 
we have to use the norm $\nabla_g ^T (X_g^TX_g)^{-1} \nabla_g$, which is 
$\Vert \nabla_g \Vert_2^2$ if and only
if each feature group $g$ is whitened, which is an assumption 
in group OMP analysis by \citet{gomp, log_gomp}. Our analysis 
sheds light  on why this assumption is important in a group setting. 
Like previous analyses of greedy algorithms by \cite{streeter:08},
our analysis guarantees that our methods produce near-optimal linear predictions, measured by explained variance, 
at budgets where feature groups are selected. Thus, they exhibit the desired anytime behavior at those budgets. Finally, 
we extend our algorithm to account for \textit{all} budgets and show 
a novel anytime result: for any budget $B$, 
if \textit{OPT} is the optimal explained 
variance of cost $B$, then our proposed sequencing can approximate 
within a factor of \textit{OPT} with cost at most $4B$. 
Furthermore, with a cost less than $4B$,
a fixed sequence of predictors cannot approximate \textit{OPT} in general. 
To our knowledge, these are the first  anytime performance bounds
at all budgets.

In previous works, both FR and OMP are theoretically analyzed for both 
the problem of subset selection and model recovery. 
\cite{kemp} cast the subset selection problem as a submodular 
maximization that 
selects a set $S$ with $|S| \leq k$ to maximize 
the explained variance and prove 
that FR and OMP achieve $(1-e^{-\lambda^*})$ and $(1-e^{-{\lambda^*}^2})$ 
near-optimal explained variance, where $\lambda^*$ is the minimum eigenvalue of the sample covariance, $\frac{1}{n}X^TX$.
We can adopt these previous analyses to our extensions to FR and OMP under
the group setting with costs and produce
the same near-optimal results. We also present a novel analysis of 
OMP that leads to the same near-optimal factor $(1-e^{-\lambda^*})$ as that of FR.
Works on model recovery have also analyzed FR and OMP. \cite{zhang:2009} proves 
that OMP discovers the true linear model coefficients, if they exist. 
This result 
was then extended by \citep{gomp, log_gomp} to the setting of feature 
groups using generalized linear models. However, we note that these
theoretical analyses of model recovery 
assume that a true model exists. They focus on recovering
model coefficients rather than directly analyzing prediction
performance.

Besides greedy selection, another family of approaches to
find the optimal subset $S$ that minimizes $R(S)$ is to
relax the NP-hard selection problem as a convex optimization. 
Lasso \citep{lasso}, a well-known method, uses $L_1$ regularization
to force sparsity in the linear model. To get an ordering of the
features, compute the Lasso solution path by varying 
the $L_1$ regularization constant. Group Lasso \citep{group_lasso} extends Lasso to the group setting, replacing the $L_1$ norm with 
the sum of $L_2$ norms of feature groups. Group Lasso can also 
incorporate feature costs by scaling the
$L_2$ norms of feature groups. 
Lasso-based methods are generally analyzed for model recovery,  
not prediction performance. We demonstrate experimentally
that our greedy methods achieve better prediction
performance than cost-weighted Group Lasso.

Various works have addressed anytime prediction previously. 
The most well-known family of approaches 
use \textit{cascades} \citep{cascade}, which achieve 
anytime prediction by filtering out samples 
with a sequence of classifiers of increasing complexity 
and feature costs. 
At each stage, cascade methods 
\citep{sochman:05, brubaker:07, lefakis:10, xu:14, cai:15} 
typically achieve a target accuracy and assign a portion of samples
with their final predictions. While this design frees up computation for 
the more difficult samples, it prevents recovery from early 
mistakes. Most cascade methods select features of each 
stage before being trained. Although the more recent works start to learn feature sequencing, the learned sequences
are the same as those of cost-weighted Group Lasso \citep{chen:12} 
and greedy methods \citep{cai:15} when they are 
restricted to linear prediction. Hence our study of anytime 
linear prediction can help cascade methods choose features and
learn cascades. 
Another branch of anytime prediction methods uses boosting. It outputs as results partial sums of the ensemble \citep{speedboost} or averages of randomly sampled weak learners \citep{reyzin:11}. Our greedy methods can be 
viewed as a gradient boosting scheme by treating each feature 
as a weak learner. 
Some works approach anytime prediction with feature transformations \citep{xu:12, xu:13b} and learn cost-sensitive, non-linear transformation of features for linear classification. Similarly, \cite{weinberger09feature} hashes high dimensional features to low dimensional subspaces. These approaches operate on readily-computed features, which is orthogonal to our problem setting. 
\cite{timeliness} models the anytime prediction as a Markov Decision Process and learns a policy of applying intermediate learners and computing features through reinforcement learning.

\paragraph*{Contributions}
\begin{itemize}[leftmargin=*]
\setlength\itemsep{1em}
\item We cast the problem of anytime linear prediction 
as a feature group sequencing problem  
and propose extensions to FR and OMP under the setting where features are in
groups that have costs. 
\item We theoretically analyze our extensions to FR and OMP 
and show that they both achieve $(1-e^{-\lambda^*})$ near-optimal 
explained variance with linear predictions at budgets when 
they choose feature groups.
\item We develop the first anytime algorithm 
that provably approximates the optimal performance
of \textit{all} budgets $B$ with cost of $4B$; we also prove it 
impossible to achieve a constant-factor approximation with cost less than $4B$. 
\end{itemize}

\section{COST-SENSITIVE GREEDY METHOD}
\label{sec:method}

This section formally introduces our extensions to FR and OMP to 
the group setting with costs. 
We assume that all feature dimensions and responses are normalized to 
have zero mean and unit variance. 
We define the regularized feature covariance matrix as 
$C := \frac{1}{n}X^TX + \lambda I_D$. Let $C_{st}$ be the sub-matrix that selects rows from $s$ and columns from $t$. Let $C_S$ be short for $C_{SS}$. 
Given a non-empty union of selected feature groups $S$, the maximum explained variance 
$F(S)$ is achieved with the regularized optimal 
coefficient 
\mbox{$w(S) = \frac{1}{n}(\frac{1}{n}X_S^TX_S + \lambda I)^{-1}(X_S^TY) = 
    \frac{1}{n} C_S^{-1}X_S^TY$}.
When we take gradient of $F(S)$ with respect to the coefficient 
of a feature group $g$, if $g \subseteq S$ then the gradient is
\mbox{$\nabla_g F(S) = \frac{1}{n} X_g^T(Y-X_Sw(S)) - \lambda w(S)_g
$}; if 
$g \cap S  =\emptyset$ then we can extend $w(S)$ to dimensions of $g$, setting $w(S)_g = 0$, and then take the gradient to have 
\mbox{$\nabla_g F(S) =\frac{1}{n} X_g^T(Y-X_Sw(S))$}. In both cases,
we have $\nabla_g F(S) = \frac{1}{n} X_g^TY - C_{gS}w(S)$. We further shorten 
the notations by defining $b_g^{S} = \nabla _g F(S)$. 
If $S$ is empty, we assume that coefficient $w(\emptyset)$ has zero for all features so that $F(\emptyset) = 0$. 
When $S = s_1, s_2,...,$ is a sequence of feature groups, we define
$S_j$ to be the prefix sequence $s_1, s_2,..., s_j$. We 
overload notations of a sequence $S$ so that $S$ also represents 
the union of its groups in notations such as $F(S)$, $w(S)$, $C_S$ and $b_S^S$.

\label{sec:algo_description}

\IncMargin{1em}
\begin{algorithm}
\caption{Cost Sensitive Group Orthogonal Matching Pursuit (CS-G-OMP)}
  \label{algo:gomp_lm}
  \SetKwInOut{Input}{input}\SetKwInOut{Output}{output}

  \Input{The normalized feature matrix $X \in \R^{n \times D}$.
    The normalized response vector $Y \in \mathbb{R}^{n}$, which has 
    a zero mean and unit variance. 
    Feature groups $\mathcal{G}_1, ... \mathcal{G}_J$ that
    partition $\{1,..,D\}$, and group costs $c(\mathcal{G}_j)$.
    Regularization constant $\lambda$.
  }
  \Output{
    A sequence $G = g_1, g_2, ..., g_{J}$ of feature groups.
    For each $j \leq J$, a coefficient 
        $w(G_j)$ for the prefix sequence $G_j = g_1,..., g_j$. 
  }

  $G_0 = \emptyset$\;
  \For{$j = 1, 2, ..., J$}{
  	\tcp{Learn linear model}
    compute $w(G_{j-1}) = \frac{1}{n}C_{G_{j-1}}^{-1}X_{G_{j-1}}^TY$\;
    \tcp{Selection step (*)}
    For each $g \notin G_{j-1}$, compute
        \quad $b_g = \nabla _g F(G_{j-1}) = \frac{1}{n} X_g^T(Y-X_{G_{j-1}}w)$ \;
    $g_{j} = \argmax \limits_{g = \mathcal{G}_1, ... , \mathcal{G}_J, g\notin G_{j-1}}  
    \frac{b_g ^T (X_{g}^TX_g)^{-1}b_g}{ c(g) }$\;
    $G_{j} = G_{j-1} \oplus g_{j}$\;
  }
  compute $w(G_J)$\;
\end{algorithm}

In Algorithm~\ref{algo:gomp_lm}, we present Cost-Sensitive Group Orthogonal Matching Pursuit (CS-G-OMP), which learns a near-optimal sequencing of the
feature groups for anytime linear predictions. 
The feature groups are selected greedily. At the $j^{th}$ selection step $(*)$, we have chosen $j-1$ groups,
$G_{j-1} = g_1, g_2, ..., g_{j-1}$, and have computed 
the best model using $G_{j-1}$, $w(G_{j-1})$. 
To evaluate a feature group $g$,
we first compute the gradient $b_g = \nabla _g F(G_{j-1})$ of the 
explained variance $F$ with respect to the coefficients of $g$.
Then, we evaluate it with the whitened gradient $L_2$-norm square per unit cost, 
$\frac{b_g ^T (X_{g}^TX_g)^{-1}b_g}{ c(g) }$. We select the group $g$ that 
maximizes this value as $g_j$, and continue until all groups
are depleted. At test time, our proposed anytime prediction algorithm computes the feature groups in the order of $G = g_1, g_2, ..., g_J$. After each feature group $g_j$ is available, we can compute and store prediction $\hat{y} = x^Tw(G_j)$ because we assumed that 
the costs of feature generation dominate the computations of linear predictions. At interruption,
we can then report the latest prediction $\hat{y}$. 

The learning procedure extending from Forward Regression is similar to
Algorithm~\ref{algo:gomp_lm}: we compute the linear models 
$w(G_{j-1} \oplus g)$ at line 4 instead of the 
gradients $b_g$ and replace the selection criterion $\frac{b_g ^T (X_{g}^TX_g)^{-1}b_g}{ c(g) }$ at line 5 with the marginal gain in explained 
variance per unit cost, $\frac{F(G_{j-1} \oplus g) - F(G_{j-1}) }{c(g)}$. We call this cost-sensitive FR extension as CS-G-FR.

Before we theoretically analyze our greedy methods
in the next section, we provide an argument 
why \textbf{group whitening} at line 5 of Algorithm~\ref{algo:gomp_lm} 
is natural. OMP greedily selects features whose
coefficients have the largest gradients of the objective function. 
In linear regression, the gradient for a feature $g$ 
is the inner-product of $X_g$ and the prediction 
residual $Y-\hat{Y}$. Hence OMP selects features that best reconstruct the 
residual. From this perspective, OMP under group setting 
should seek the feature group whose span contains the largest projection of the residual. 
Let the projection to feature group $g$ be 
$P_g = X_g(X_g^TX_g)^{-1}X_g^T$ and recall projection matrices are
idempotent. We observe that the criterion for CS-G-OMP selection step is
$\frac{\Vert P_g (Y - \hat{Y}) \Vert _2^2 }{c(g)}$, i.e, a cost-weighted
norm square of the projection of the residual onto a feature group. The 
name group whitening is chosen because the criterion is 
$\frac{\Vert b_g \Vert_2^2}{c(g)}$ if and only if feature groups are whitened. We assume \textit{feature groups are whitened} in our formal analysis to make the criterion easier to analyze. 

Besides the above greedy criterion, one may suggest other approaches 
to evaluate gradient vectors $b_g$ for group $g$. For example, 
$L_2$ norm and $L_{\infty}$ norm can be used to 
achieve greedy criteria $\frac{\Vert b_g \Vert_2^2}{c(g)}$ and 
$\frac{\Vert b_g \Vert ^2_{\infty}}{c(g)}$, respectively. 
The former criterion forgoes group whitening, so we call it \textit{no-whiten}.
Thus, it overestimates a feature group that has correlated 
but effective features, an extreme example of which is a
feature group of identical but effective features. The latter
criterion evaluates only the best feature of each feature group, so we call it \textit{single}. Thus, it
underestimates a feature group that has a descriptive 
feature span but no top-performing individual feature dimensions.
We will show in experiments that no-whiten and single
are indeed inferior to our CS-G-OMP choice. 



\section{THEORETICAL ANALYSIS}
\label{sec:proof}

This section proves that CS-G-FR and CS-G-OMP 
produce near-optimal explained variance $F$ at budgets 
where features are selected. The main challenge of our analysis is to prove Lemma~\ref{lemma:main},
which is a common stepping stone in 
submodular maximization analysis, e.g., Equation 8 in \citep{submodular}. The main Theorem~\ref{thm:main} follows from the lemma by standard techniques, which we defer to the appendix. 

\begin{lemma}[main]
  Let $G_j$ be the first $j$ feature groups selected by our greedy algorithm. There exists a constant $\gamma = \frac{\lambda^* + \lambda}{1 +\lambda} > 0$ such that for any sequence $S$, total cost $K$, and indices $j=1,2,..., J$, 
  \mbox{$
    F(S_{\angleb{K}}) - F(G_{j-1}) \leq \frac{K}{\gamma}
      \lbrack \frac{F(G_j) - F(G_{j-1})}{c(g_j)} \rbrack.
  $}
  \label{lemma:main}
\end{lemma}

\begin{theorem}
Let $B = \sum _{i=1}^L c(g_i)$ for some $L$.  
There exists a constant  
  $\gamma = \frac{\lambda^* + \lambda}{1+\lambda}$, 
  such that
for any sequence $S$ and total cost $K$, 
\mbox{$
  F(G_{\angleb{B}}) > (1 - e^{-\gamma\frac{B}{K}})F(S_{\angleb{K}}).
$}
\label{thm:main}
\end{theorem}

%
Before delving into the proof of Lemma~\ref{lemma:main}, we first discuss 
some implications of Theorem~\ref{thm:main}, which 
argues that the explained variance of greedily selected
features of cost $B$ is within $(1-e^{\gamma \frac{B}{K}})$-factor
of that of any competing feature sequence of cost $K$.
If we apply minimum regularization $(\lambda \rightarrow 0)$, then 
the constant $\gamma$ approaches $\lambda^*$. The resulting bound factor $(1-e^{ - \lambda^* \frac{B}{K}})$ is the bound for FR by \cite{kemp}. However, we achieve the same bound for OMP, improving
theoretical guarantees of OMP. We also note that less-correlated features lead
to a higher $\lambda^*$  and a stronger bound.


Lemma~\ref{lemma:main} for CS-G-FR is standard if we follow proofs in \citep{streeter:08} and \citep{kemp} because the objective $F$ is $\gamma$-approximately submodular. 
However, we present a proof of 
Lemma~\ref{lemma:main} for CS-G-OMP without approximate submodularity to achieve the same constant $\gamma$. 
This proof in turn uses Lemma~\ref{lemma:smoothness} and Lemma~\ref{lemma:convexity}, whose proofs are based on the Taylor expansions of the regularized risk $\mathcal{R}[f_S]=R(S)$, a $M$-strongly smooth and $m$-strongly convex loss functional of predictors $f(x) = w^T x$.
We defer these two proofs to the appendix and note that 
$M=m$ with our choice of $R$.

\begin{lemma}[Using Smoothness]
  Let $S$ and $G$ be some fixed sequences. Then
  \mbox{$
    F(S) - F(G) \leq \frac{1}{2m} \angleb{b^G_{G \oplus S}, C_{G \oplus S}^{-1} b^G_{G\oplus S}}.
  $}
  \label{lemma:smoothness}
\end{lemma}

\begin{lemma}[Using Convexity] For $j = 1,2,..., J$, 
    \mbox{$
      F(G_j) - F(G_{j-1}) \geq \frac{1}{2M} \angleb{ {b^{G_{j-1}}_{g_j}}, C_{g_j}^{-1}b^{G_{j-1}}_{g_j} }.
    $}
  \label{lemma:convexity}
\end{lemma}
Note that in Lemma~\ref{lemma:convexity}, since we assume feature groups are 
whitened, then $C_{g_j} = (1+\lambda) I$. The bound of the lemma becomes
$F(G_j) - F(G_{j-1}) \geq \frac{1}{2M (1+\lambda)} \angleb{ {b^{G_{j-1}}_{g_j}}, b^{G_{j-1}}_{g_j} }$. If feature groups are not whitened, 
the constant $(1+\lambda)$ can be scaled up to $(|\mathcal{G}_j| + \lambda)$, 
which detriments the strength of Theorem~\ref{thm:main} especially when feature 
groups are large.

\begin{proof} (of Lemma~\ref{lemma:main}, using Lemma~\ref{lemma:smoothness} and Lemma~\ref{lemma:convexity}) \\
  Using Lemma ~\ref{lemma:smoothness}, on $S_{\angleb{K}}$ and $G_{j-1}$, we have: 
  \begin{align}
    &F(S_{\angleb{K}}) - F(G_{j-1})  \notag \\
    &\leq 
      \frac{1}{2m} \angleb{b^{G_{j-1}}_{G_{j-1} \oplus S_{\angleb{K}}},
      C^G_{G_{j-1} \oplus S_{\angleb{K}}} b^{G_{j-1}}_{G_{j-1} \oplus S_{\angleb{K}}}}
  \end{align}
  Note that the gradient $b_{G_{j-1}}^{G_{j-1}}$  
  		equals $0$, because $F(G_{j-1})$ is achieved by  
  		the linear model $w(G_{j-1})$. Then, using block matrix inverse
  formula, we have:
  \begin{align}
    F(S_{\angleb{K}}) - F(G_{j-1}) \leq 
      \frac{1}{2m} 
      \angleb{b^{G_{j-1}}_{S_{\angleb{K}}},
      C^G_{S_{\angleb{K}}} 
      b^{G_{j-1}}_{S_{\angleb{K}}}}
  \end{align}
  where $
    C^G_{S_{\angleb{K}}} = C_{S_{\angleb{K}}} - C_{{S_{\angleb{K}}}G} 
      C^{-1}_{S_{\angleb{K}}} C_{G{S_{\angleb{K}}}}.  $
  Using spectral techniques in Lemmas 2.5 and 2.6 in \citep{kemp} and
  noting that the minimum eigenvalue of $C$, $\lambda_{min}(C)$, is $\lambda^* + \lambda$, we have
  \begin{align}
      \frac{1}{2m} 
      \angleb{b^{G_{j-1}}_{S_{\angleb{K}}}, 
      C^G_{S_{\angleb{K}}} 
      b^{G_{j-1}}_{S_{\angleb{K}}}}
    \leq 
      \frac{1}{2m (\lambda^* + \lambda)} 
     \angleb{b^{G_{j-1}}_{S_{\angleb{K}}}, 
      b^{G_{j-1}}_{S_{\angleb{K}}}}.
  \end{align}
  Expanding $S_{\angleb{K}}$ into individual groups $s_i$, we continue:
  \begin{align}
    &= 
    	\frac{1}{2m(\lambda^* + \lambda)} \sum _{s_i \in S_{\angleb{K}}} 
       \angleb{b^{G_{j-1}}_{s_i}, {b^{G_{j-1}}_{s_i}}}  \\
    &\leq
        \frac{1}{2m(\lambda^* + \lambda)} \sum _{s_i \in S_{\angleb{K}}} 
        c(s_i) \max_{g} \frac{  \angleb{b^{G_{j-1}}_{g}, {b^{G_{j-1}}_{g}}}}{c(g)} \\
    &=
        \frac{1}{2m(\lambda^* + \lambda)} \sum _{s_i \in S_{\angleb{K}}} 
        c(s_i) \frac{\angleb{b^{G_{j-1}}_{g_j}, {b^{G_{j-1}}_{g_j}}}}{c(g_j)} \\
    &\leq
        \frac{ M(1+ \lambda)}{m (\lambda^* + \lambda)} \sum _{s_i \in S_{\angleb{K}}} 
        c(s_i)
          \frac{ F(G_{j}) - F(G_{j-1}) } { c(g_j) }.
  \end{align}
  The last equality follows from the greedy selection step of Algorithm~\ref{algo:gomp_lm} when feature groups are whitened. 
  The last inequality is given by Lemma ~\ref{lemma:convexity}. The 
  theorem then follows from $\gamma = (\frac{m}{M}) \frac{\lambda^* + \lambda}{1+\lambda} = \frac{\lambda^* + \lambda}{1+\lambda}$. 
\end{proof}

\section{BI-CRITERIA APPROXIMATION AT ALL BUDGETS}

Our analysis so far only bounds algorithm performance at 
budgets when new items are selected. However, an ideal analysis
should apply to all budgets. As illustrated in Figure~\ref{fig:all-budget-bad},
previous methods may choose expensive features early; 
until they are computed, we have no bounds. 
Figure~\ref{fig:all-budget-good} illustrates our proposed fix: each 
new item $g_{j+1}$ cannot be more costly than the current sequence $G_{j}$. 

This section proves two theorems of anytime prediction at \textit{any} budget.  Theorem~\ref{thm:greedy.biapproximation-upper-bound} shows that
 to approximate the optimal explained variance 
of cost $B$ within a constant factor,
an anytime algorithm must cost at least $4B$. 
We then motivate and formalize our fix in Algorithm~\ref{alg:greedy.doubling},
which is shown in
Theorem~\ref{thm:greedy.doubling-greedy-bound-approx} to achieve this
\textit{bi-criteria approximation} bound for both budget and objective with
the form: \mbox{$F(G_{\at{B}}) > (1 - e^{-\frac{\gamma^2}{1+\gamma}}) F(S_{\at{\frac{B}{4}}})$}, where $\gamma$ is the approximate submodular
ratio, i.e., the maximum constant $\gamma \leq 1$ such that for 
all sets $ A' \subseteq A$ and all element $x$,
\begin{equation}
\label{def:greedy.approx-submodularity}
    \gamma (F(A \cup \{x\}) - F(A)) \leq F(A' \cup \{x\}) - F(A').
\end{equation}

We first illustrate the inherent difficulty in 
generating single sequences that are competitive at arbitrary budgets
$B$ by using the following budgeted maximization problem:
\begin{align}
\label{eq:greedy.hard-problem}
X = \{1,2,\ldots\},\;\; c(x) = x, \;\;
F(S) = \sum_{x \in S} e^x.
\end{align}
The above problem originates from fitting the linear model
$Y = \sum _{i=1}^D e^iX_i$, where $X_i$'s are i.i.d. and $X_i$ 
costs $i$. 

\begin{theorem}
\label{thm:greedy.biapproximation-upper-bound}
Let $\mathcal{A}$ be any algorithm for selecting sequences $A = (a_1,
\ldots)$.  The best bi-criteria approximation that $\mathcal{A}$ can
satisfy must be at least a $4$-approximation in cost for the sequence
described in Equation~(\ref{eq:greedy.hard-problem}).  That is, there
does not exist a $C < 4$, and a $c_1 \in [0,1)$, such that for any budget $B$ and
any sequence $S$,
\[
F(A_{\at{B}}) > \left(1 - c_1\right) F(S_{\at{\frac{B}{C}}}).
\]
\end{theorem}

\begin{proof}
For any budget $B$, it is clear that the optimal selection contains
a single item, $B$, whose value is $e^B$. 
For any budget $B$, let $m(B)$ denote the item of the maximum cost that is selected by the algorithm.
If the bi-criteria bound holds, then 
$\sum _{k=1}^{m(B)} e^k \geq F(A_{\at{B}}) > \left(1 - c_1\right) F(S_{\at{\frac{B}{C}}})$. 
Taking the log of both sides and rearranging terms, we have $m(B) \geq \lfloor \frac{B}{C} \rfloor + \ln(1-c_1) + \ln(e-1) - 2$. 
Since $3 - \ln (1-c_1) - \ln (e-1) > 0$,  we have for $B$ large enough:
$C \geq \frac{B}{m(B)}. $ 
Hence, we need to minimize $\frac{B}{m(B)}$ for all $B$ to minimize $C$. We
can assume $a_j$ to be increasing 
because otherwise we could remove the violating $a_j$ 
from the sequence and decrease the ratio $\frac{B}{m(B)}$ for all subsequent $j$. 

Let $b_j := c(A_j)$ and $\alpha_{j} := \frac{c(a_{j})}{b_{j-1}}$. 
Then immediately before $a_{j}$ is available,
$\frac{B}{m(B)} \rightarrow 
\frac{c(A_j)}{c(a_{j-1})} \geq \frac{(1+\alpha_j)b_{j-1}}{b_{j-1}} = 1+\alpha_j$. If we can bound $\frac{B}{m(B)}\leq C$ for all $B$, 
then there exists $\alpha_{max}$ such that
$\alpha_j < \alpha_{max}$ for all $j$ large enough. 
Immediately after a new $a_j$ is selected, 
$\frac{B}{m(B)} = \frac{c(A_j)}{c(a_{j})} = \frac{1+\alpha_j}{\alpha_j}$. 
For $\frac{B}{m(B)}$ to be bounded, there must exist some $\alpha_{min}>0$ such that
$\alpha_j > \alpha_{min}$ for large enough $j$.
Now we consider the ratio $\frac{B}{m(B)}$ right before $a_{j+1}$ is selected:
\begin{align}
\hspace{-13pt} \frac{c(A_{j+1})}{c(a_{j})}= \frac{b_j(1+\alpha_{j+1})}{b_j\frac{\alpha_{j}}{1+\alpha_j}} =  
1 + \frac{\alpha_{j+1}}{\alpha_j} + \alpha_{j+1} + \frac{1}{\alpha_j}.
\label{line:lower_bound_cost_ratio}
\end{align}
Assume for seek of contradiction that $\frac{c(A_{j+1})}{c(a_{j})}$ 
is bounded above by $z$ for some $z \in (1,  4)$.
Let $y := \frac{\alpha_{j+1}}{\alpha_j}$. Then we have:
$ z \geq 1 + y + y\alpha_j + \frac{1}{\alpha_j} \geq 1 + y + 2\sqrt{y} = (\sqrt{y} + 1)^2$. Hence $y \leq (\sqrt{z} -1)^2 < 1$. 
So \mbox{$a_{j+1} \leq (\sqrt{z} -1)^2 a_j$}, which implies that $a_j$ converges to $0$ and we have a contradiction.
So \mbox{$C \geq \frac{B}{m(B)}  \rightarrow \frac{c(A_{j+1})}{c(a_{j})} \geq 4$} for large $j$.
\end{proof}

The above proof lower bounds the cost approximation ratio $C$ by Eq.~\ref{line:lower_bound_cost_ratio}, which is shown to be at least $4$ for $C < \infty$. We note that $Eq.~\ref{line:lower_bound_cost_ratio}$ equals $4$ if $\forall j, \alpha_j = 1$, which means the sequence total cost is doubled at each selection step.
This observation leads to \textit{Doubling Algorithm} (Alg.~\ref{alg:greedy.doubling}): we perform greedy selection in the same way as CS-G-FR, except that the total cost can be at most doubled at each step (illustrated in Figure~\ref{fig:doubling-algo}). 
The advantage of Doubling Algorithm over 
CS-G-FR is that 
the former prevents early computation of expensive features and induces a smoother increase of total cost; in most real-world data-sets, the two are identical after few steps because 
feature costs are often in a narrow range. 
We will analyze Doubling Algorithm with the following assumption, called \textit{doubling capable}.

\begin{figure}
\centering
\subfloat[Before $F$ is computed, we have no output or bounds.]{
  \includegraphics[width=0.49\textwidth]{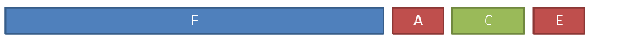}
  \label{fig:all-budget-bad}
}

\subfloat[Our constraint $c(g_{j+1}) \leq c(G_j)$ induces a smoother cost increase. ] {
  \includegraphics[width=0.49\textwidth]{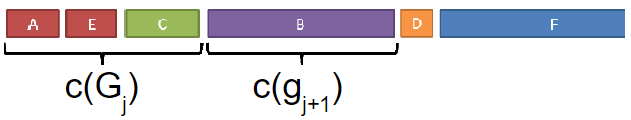}
  \label{fig:all-budget-good}
}

\subfloat[Illustration of Doubling Algorithm Cost Constraint]{
\includegraphics[width=0.45\textwidth]{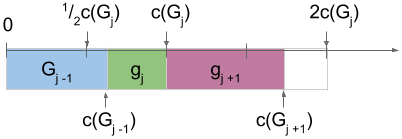}
  \label{fig:doubling-algo}
}

\caption{Doubling Algorithm (b) has better anytime behaviors 
than greedy algorithm with no cost constraints (a).}
\label{fig:doubling}
\end{figure}

\begin{algorithm}[t]
  \caption{Doubling Algorithm}
  \label{alg:greedy.doubling}
  
  \SetKwInOut{Input}{input}
    \Input{ objective function $F$, elements $X$, minimum cost $c_{\textrm{min}}$}
    
    Let $g_1 = \underset{x \in X,\ c(x) \le c_{\textrm{min}}}{\argmax} \frac{F(\{x\})}{c(x)}$; 
    Let $G_1 = g_1$ \;
    \For{$j = 2,\ldots$} {
    Let $g_j = \underset{x \in X \setminus G_{j-1},\ c(x) \le c(G_{j-1})}{\argmax} \;
      \frac{F(G_{j-1} \oplus \{x\}) - F(G_{j-1})}{c(x)}$ \;
    Let $G_j = G_{j-1} \oplus \{g_j\}$ \;
    }
\end{algorithm}

\begin{definition} 
\label{def:greedy.doubling-capable}
Let $G = (g_1, \ldots)$ be the sequence selected by the doubling
algorithm.  The set $X$ and function $F$ are \textit{doubling capable}
if, at every iteration $j$, the following set is non-empty:
$
\{x \mid x \in X \setminus G_{j-1},\ c(x) \le c(G_{j-1})\}
$
\end{definition}

\begin{theorem}
\label{thm:greedy.doubling-greedy-bound-approx}
Let $G = (g_1, \ldots)$ be the sequence selected by the doubling
algorithm (Algorithm~\ref{alg:greedy.doubling}).  Fix some $B >
c_{\textrm{min}}$.  Let $F$ be $\gamma$-approximately
submodular as in Definition~\ref{def:greedy.approx-submodularity}.
For any sequence $S$,
\[
F(G_{\at{B}}) > \left(1 - e^{-\frac{\gamma^2}{1+\gamma}} \right) F(S_{\at{\frac{B}{4}}}).
\]
\end{theorem}
\begin{proof}
Doubling capable easily leads to the observation that for all budgets $B$, there exists an index $j$ such that
$\frac{B}{2} \leq c(G_j) < B$.
Choose $K$ and $k$ to be  the largest integers such that
$\frac{B}{2} \leq c(G_K) < B$ and 
$\frac{B}{8} \leq c(G_k) < \frac{B}{4}$. Since at each step we at most double the 
total cost and $4c(G_k) < B$, we observe $K \geq k+2$. 
For each $j$, define $s_j = \frac{F(G_{j+1}) - F(G_{j})}{c(g_{j+1})}$ as the best 
rate of improvement among the items Doubling Algorithm is allowed to consider
after choosing $G_j$. 
Consider the item $x$ in sequence $S_{\at{\frac{B}{4}}}$ of the maximum cost. 

(Case 1) If $c(x) \leq c(G_k)$, then
every item in $S_{\at{\frac{B}{4}}}$ was a candidate for $g_{j}$ for all $j=k+1,..., K$. 
So by approximate submodularity from Equation~\ref{def:greedy.approx-submodularity}, we have 
\begin{align}
\label{eq:sub-additive}
F(S_{\at{\frac{B}{4}}}) \leq F(S_{\at{\frac{B}{4}}} \cup G_{j}) \leq F(G_j) + \frac{B s_j}{4 \gamma}.
\end{align}
Then using the standard submodular maximization proof technique, we define
\mbox{$\Delta _j = F(S_{\at{\frac{B}{4}}}) - F(G_j)$}. Applying $s_j = \frac{\Delta _{j} - \Delta_{j+1}}{c(g_{j+1})}$ in the above inequality, 
we have
\mbox{$\Delta _{k+j} \leq \Delta_k \prod _{j=k+1}^{k+j} ( 1 - \gamma \frac{ 4 c(g_{j})}{B})$}. Maximizing the 
inequality by setting $c(g_{j}) = \frac{B}{K-k} \leq \frac{c(G_K) - c(G_k)}{4 (K-k)}$, 
and using $(1- z/l)^l < e^{-z}$, we have 
\mbox{$F(G_K) > (1 - e^{-\gamma}) F(S_{\at{\frac{B}{4}}}).
$}

From now on, we assume that $c(x) > c(G_k)$ and consider 
two cases by comparing $c(g_{k+2})$ and $c(G_{k})$. 

(Case 2.1) If $c(g_{k+2}) \geq c(G_{k})$, then 
$c(G_K) - c(G_{k+1}) \geq c(g_{k+2}) \geq c(G_k)$. 
Since $c(G_{k+1}) \leq 2 c(G_k)$ and $c(x) > c(G_k)$, 
we have $c(G_K) - c(G_{k+1}) \geq \frac{B}{2} - 2c(G_k)$.
So \mbox{$c(G_K) - c(G_{k+1}) \geq \max ( c(G_k), \frac{B}{2} - 2c(G_k) ) \geq \frac{B}{6}$}. 
Thus, using the same proof techniques as in case 1, we can analyze the ratio between $\Delta_{k+1}$ and $\Delta_K$, and have:
\mbox{$
F(G_K) > (1 - e^{-\frac{2}{3} \gamma}) F(S_{\at{\frac{B}{4}}}).
$}

(Case 2.2) 
Finally, if \mbox{$c(g_{k+2}) < c(G_k) < c(x) < c(G_{k+1})$},
$g_{k+2}$ was a candidate for $g_{k+1}$, and $x$ was a candidate for 
$g_{k+2}$. 
For an item $y$, let 
\mbox{$r(y^j)= \frac{F(G_{j} \cup \{ y \}) - F(G_j)}{c(y)}$} 
be the improvement rate of item $y$ at $G_j$. 
Then we have \mbox{$r(g_{k+1}^k) > r(g_{k+2}^k)$} and \mbox{$r(g_{k+2}^{k+1}) > r(x^{k+1})$}. 
Since the objective function is increasing, we have 
\mbox{$r(x^k) c(x) \leq r(x^{k+1})c(x) + r(g_{k+1}^k)c(g_{k+1})$},
so that 
\mbox{$r(x^k) \leq r(x^{k+1}) + r(g_{k+1}^k) \frac{c(g_{k+1})}{c(x)}$}.
Then by the definition of $\gamma$ in Equation~\ref{def:greedy.approx-submodularity}, we have 
$ \gamma r(g^{k+1}_{k+2}) \leq r(g_{k+2}^k)$. Hence we have 
$ \gamma r(x^{k+1}) \leq  r(g_{k+1}^k)$, which leads to 
\mbox{$r (x^k) \leq r(g_{k+1}^k) (\frac{1}{\gamma} + \frac{c(g_{k+1})}{c(x)} ) \leq r(g_{k+1}^k) (1 + \frac{1}{\gamma})$}. Then
inequality~(\ref{eq:sub-additive}) holds with a coefficient adjustment and becomes
$
F(S_{\at{\frac{B}{4}}}) \leq F(G_k) + \frac{B s_k (1+\gamma)}{4 \gamma^2}.
$
Noting that the above inequality holds for all $j=k+1, ..., K$, we can replace the constant $\gamma$ in the 
proof of case $1$ with $\frac{\gamma^2}{1+\gamma}$ and have the following bound:\mbox{
$
F(G_K) > (1 - e^{-\frac{\gamma^2}{1+\gamma} }) F(S_{\at{\frac{B}{4}}}).
$}

\end{proof}


\section{EXPERIMENTS}
\label{sec:experiment}

\begin{table*}
\caption{Test time 0.97-Timeliness measurement of different methods on \Grain. We break
 the methods into OMP, FR and Oracle family: e.g., ``Single" in the G-CS-OMP family means G-CS-OMP-Single, and ``FR" in the Oracle family means the oracle curve derived from G-FR. }
\label{tab:grain_auc} 
\begin{center}
\resizebox{0.7\textwidth}{!}{
\begin{tabular}{cccc|c|cc|c}
\multicolumn{4}{c|}{CS-G-OMP-Variants} &
\multicolumn{1}{c|}{CS-G-FR} &
\multicolumn{2}{c|}{Oracles} &
\multicolumn{1}{c}{Sparse} \\
	CS-G-OMP & 
	Single & 
	No-Whiten & 
	G-OMP & 
	\; & 
	FR Oracle &
	OMP Oracle &
  \; \\
  \hline
	\textbf{0.4406} & 
	0.4086 &
	0.4340 &
	0.4073 &
	\textbf{0.4525} &
	\textbf{0.4551} & 
	0.4508 & 
  0.3997 \\
\end{tabular}
}
\end{center}
\end{table*}

\begin{table*}
\caption{Test time 0.99-Timeliness measurement of different methods on \YahooLTR.}
\label{tab:yahoo_auc} 
\begin{center}
\resizebox{0.7\textwidth}{!}{
\begin{tabular}{c|cccc|c|cc|c}
Group &
\multicolumn{4}{c|}{CS-G-OMP-Variants} &
\multicolumn{1}{c|}{CS-G-FR} &
\multicolumn{2}{c|}{Oracles} &
\multicolumn{1}{c}{Sparse} \\
	Size &
	CS-G-OMP & 
	Single & 
	No-Whiten & 
	G-OMP & 
	\; & 
	FR &
	OMP &
  \; \\
\hline
	5      & 
	\textbf{0.3188} & 
	0.3039 &
	0.3111 &
	0.2985 &
	\textbf{0.3222} &
	\textbf{0.3225} & 
	0.3211 & 
  0.2934 \\

	10      & 
	\textbf{0.3142} & 
	0.3117 &
	0.3079 &
	0.2909 &
	\textbf{0.3205} &
	\textbf{0.3207} & 
	0.3164 & 
  0.2858 \\

	15      & 
	\textbf{0.3165} & 
	0.3159 &
	0.3116 &
	0.2892 &
	\textbf{0.3213} &
	\textbf{0.3213} & 
	0.3177 & 
  0.2952 \\

	20     & 
	\textbf{0.3161} & 
	0.3124 &
	0.3065 &
	0.2875 &
	\textbf{0.3180} &
	\textbf{0.3180} & 
	0.3163 & 
  0.2895 \\

\end{tabular}
}
\end{center}
\end{table*}

\begin{figure}
\centering
\subfloat[Training Time OMP vs. FR (\Grain)]{
  \includegraphics[width=0.37\textwidth,height=3.5cm]{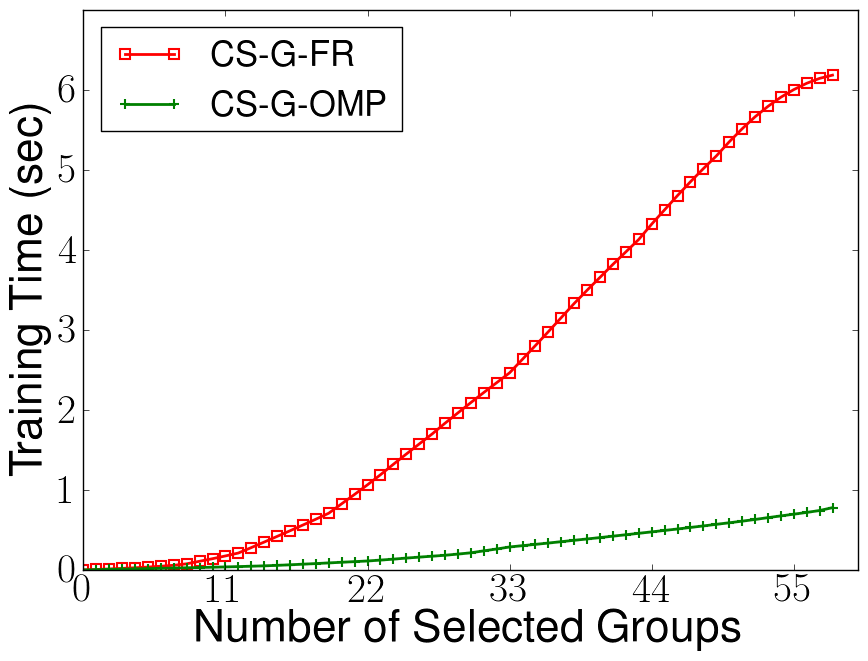}
}

\subfloat[Training Time OMP vs. FR (\YahooLTR)]{
  \includegraphics[width=0.37\textwidth,height=3.5cm]{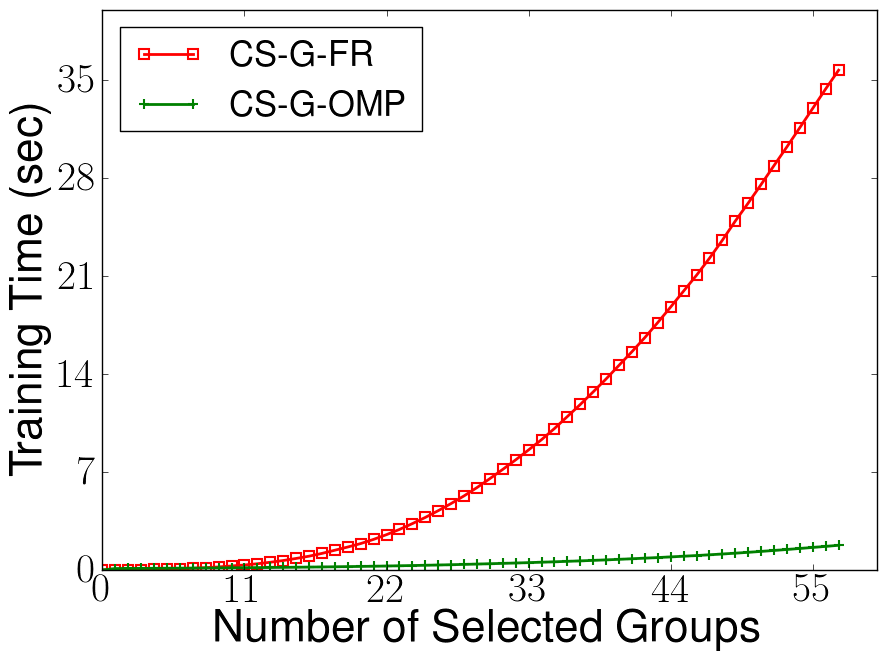}
}
\caption{The training time vs. the number of feature groups selected with two algorithms: CS-G-FR and 
CS-G-OMP. CS-G-OMP achieves a 8x and 20x overall training time speed-up 
on \Grain\, and \YahooLTR.} 
\label{fig:run_time}
\end{figure}

\subsection{DATA-SETS AND SET-UP}
We experiment our methods for anytime linear prediction on two real-world data-sets,
each of which has a significant number of feature groups with associated costs. 

\begin{itemize}[leftmargin=*]
\item \textbf{Yahoo! Learning to Rank Challenge} \citep{yahoo_ltr}
contains 883k web documents, each of which has a relevance score in $\{0, 1, 2, 3, 4\}$. Each of the 501 document features has an associated computational cost in 
$\{ 1, 5, 20, 50, 100, 150, 200\}$; the total feature cost is around 17K. The original data-set has no feature group structures, so we generated random group structures by grouping features of the same cost into groups of a given size $s$.\footnote{We experiment on group sizes $s \in \{ 5, 10, 15, 20 \}$. We choose regularizer 
$\lambda = 10^{-5}$ based on validation. We use 
$s=10$ for qualitative results such as plots and illustrations, but we report quantitative results for all group size $s$. For our quantitative results, we report the average test performance. The initial risk is $R(\emptyset)=0.85$.}

\item \textbf{Agriculture} is a proprietary data-set that contains 510k data samples, 328 features, and 57 feature groups. Each sample has a binary label in $\{1, 2\}$. Each feature group has an associated cost measured in its 
average computation time.\footnote{
There are 6 groups of size 32; the other groups have sizes between 1 and 6. 
The cost of each group is its expected computation time in seconds, ranging between 0.0005 and 0.0088; the total feature cost is 0.111. 
We choose regularizer $\lambda = 10^{-7}$. The data-set is 
split into five 100k sets, and the remaining 10k are used for validation. We report the cross validation results on the five 100K sets as the test results. The initial risk is $R(\emptyset) = 0.091$.}
\end{itemize}

\subsection{EVALUATION METRIC, BASELINE AND ORACLE}
\label{sec:timeliness}
\begin{figure}[t]
\centering
\subfloat[Plateau Effect and $\alpha$-Stopping Costs]{
 \includegraphics[width=0.40\textwidth,height=3.8cm]{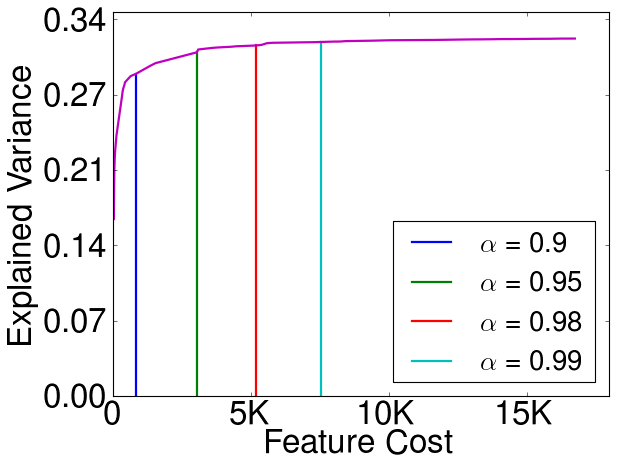}
 \label{fig:timeliness_a}
}

\subfloat[Importance of Costs (CS-G-OMP vs. G-OMP)]{
 \includegraphics[width=0.40\textwidth,height=3.8cm]{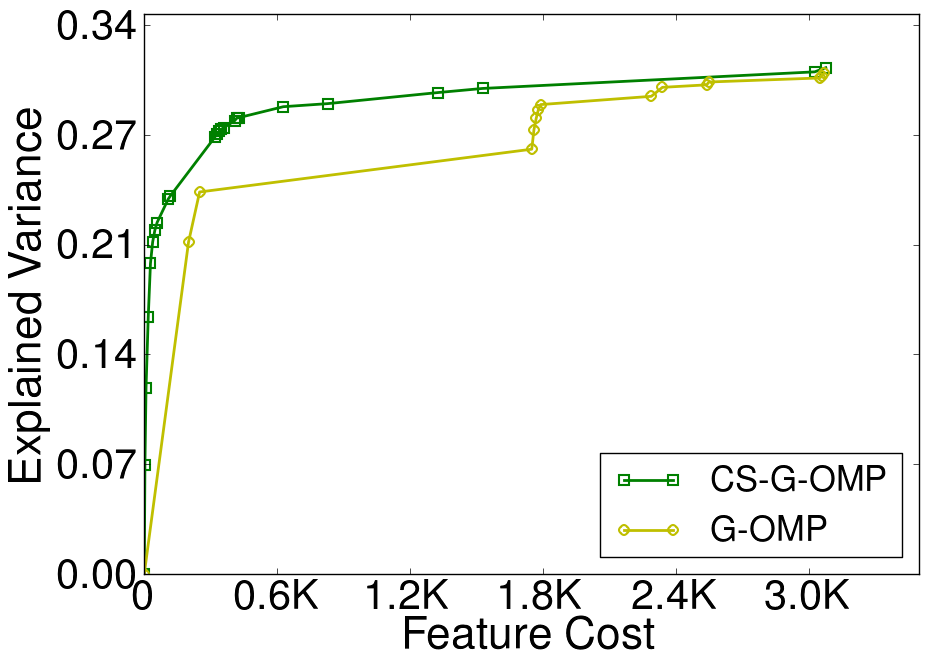}
 \label{fig:cost_vs_no_cost}
}
\label{fig:timeliness}
\caption{(a) Explained Variance vs. Cost curve of CS-G-OMP in
\YahooLTR. Vertical lines mark different $\alpha$-stopping costs.  (b) Explained Variance vs. Cost curve of CS-G-OMP and G-OMP on \YahooLTR\, set 1 with individual group size $s=10$, stopped at 0.97-stop cost.}
\end{figure}

Following the practice of \cite{timeliness}, we use the area under the 
maximization objective $F$ (explained variance) vs. cost curve normalized by the total area as the 
\textit{timeliness} measurement of the anytime performance of an algorithm\footnotetext{\cite{timeliness} define \textit{timeliness} as the area under the average precision vs. time curve}. In our data-sets, the performance of linear predictors plateaus much before
all features are used, e.g., Figure~\ref{fig:timeliness_a} demonstrates this effect in \YahooLTR, where the last one percent of total improvement is bought by half of the total feature cost. Hence the majority of the timeliness measurement is from the plateau performance of linear predictors. The difference between timeliness of different anytime algorithms diminishes due to the plateau effect. Furthermore, the difference vanishes as we include additional redundant high cost features. To account for this effect, we 
stop the curve when it reaches the plateau.
We define an \textit{$\alpha$-stopping cost} for parameter $\alpha$ in $[0,1]$ as the cost at which our CS-G-OMP achieves $\alpha$ of the final objective value in training and ignore the objective vs. cost curve after
the $\alpha$-stopping cost. We call the timeliness measure on the shortened curve 
as \textit{$\alpha$-timeliness}; 1-timeliness equals the normalized area under the full curve and 0-timeliness is zero. If a curve does not pick a group at $\alpha$-stopping cost, we linearly interpolate the objective value at the stopping cost to 
computr timeliness. 
We say an objective vs. cost curve has reached its final plateau if at least 95\% of the total 
objective has been achieved and the next 1\% requires more than 20\%
feature costs. (If the plateau does not exist, we use $\alpha = 1$.) Following this rule, we choose $\alpha = 0.97$ for \Grain\ and $\alpha = 0.99$ for \YahooLTR.

Since an exhaustive search for the best feature sequencing is intractable, 
we approximate with the \textbf{Oracle} anytime performance following the approach of \cite{timeliness}. Given an objective vs. cost curve of a sequencing, we reorder the feature groups in descending order of their marginal benefit per unit cost, assuming that the marginal benefits stay the same after reordering. We specify which sequencing is used for creating \textbf{Oracle} in Section~\ref{sec:selection_methods}. 
For baseline performance, we use cost-weighted Group Lasso \citep{group_lasso}, which
scales the regularization constant of each group with the cost of the group. We note that the cascade design by 
\cite{chen:12} can be reduced to this baseline if we enforce
linear prediction. 
More specifically, the baseline solves the following minimization problem:
\mbox{$
  \min _{w \in \mathbb{R}^{D}} \Vert Y -
    Xw \Vert^2_2 + \lambda
    \sum _{j=1}^J c(\mathcal{G}_j) \Vert w _{\mathcal{G}_j} \Vert _2,
$}
and we vary value of regularization constant
$\lambda$ to obtain lasso paths. We call this baseline algorithm \textbf{Sparse}\footnote{We use an off-the-shelf software, 
SPAMS (SPArse Modeling Software \citep{spams}), to solve the optimization.}.

\subsection{FEATURE COST}
Our proposed CS-G-OMP differs from Group Orthogonal Matching Pursuit (G-OMP) \citep{gomp} in that G-OMP does not consider feature costs when evaluating features. We show that this difference is crucial for anytime linear prediction. In Figure~\ref{fig:cost_vs_no_cost}, we compare the objective vs. costs curves of CS-G-OMP and G-OMP that are stopped at 0.97-stopping cost on \YahooLTR. As expected, CS-G-OMP achieves a
better overall prediction at every budget, qualitatively demonstrating the importance of incorporating feature costs. Table~\ref{tab:grain_auc} and Table~\ref{tab:yahoo_auc} 
quantify this effect, showing that CS-G-OMP 
achieves a better timeliness
measure than regular G-OMP. 

\subsection{GROUP WHITENING}
We provide experimental evidence that  
Group whitening, i.e., $X_g^TX_g = I_{D_g}$ for each group $g$, is a key assumption of both this work and previous feature group selection literature  by \cite{gomp, log_gomp}.
In Figure~\ref{fig:whiten_vs_no_whiten}, we compare 
anytime prediction performances using group whitened data 
against those using the common  
normalization scheme where each feature dimension
is individually normalized to have zero mean and unit variance. 
The objective vs. cost curve qualitatively shows that group whitening consistently results in the better predictions.
This behavior is expected from data-sets whose feature groups contain correlated features, e.g., group whitening effectively prevents selection step $(*)$ from overestimating the predictive power of feature groups of repeated good features. Table~\ref{tab:grain_auc} and Table~\ref{tab:yahoo_auc} demonstrate quantitatively the consistent better timeliness performance of CS-G-OMP over that of CS-G-OMP-no-whiten.

\begin{figure}
\centering
\subfloat[Group Whiten vs. No-Whiten (\Grain)]{
  \includegraphics[width=0.4\textwidth,height=3.8cm]{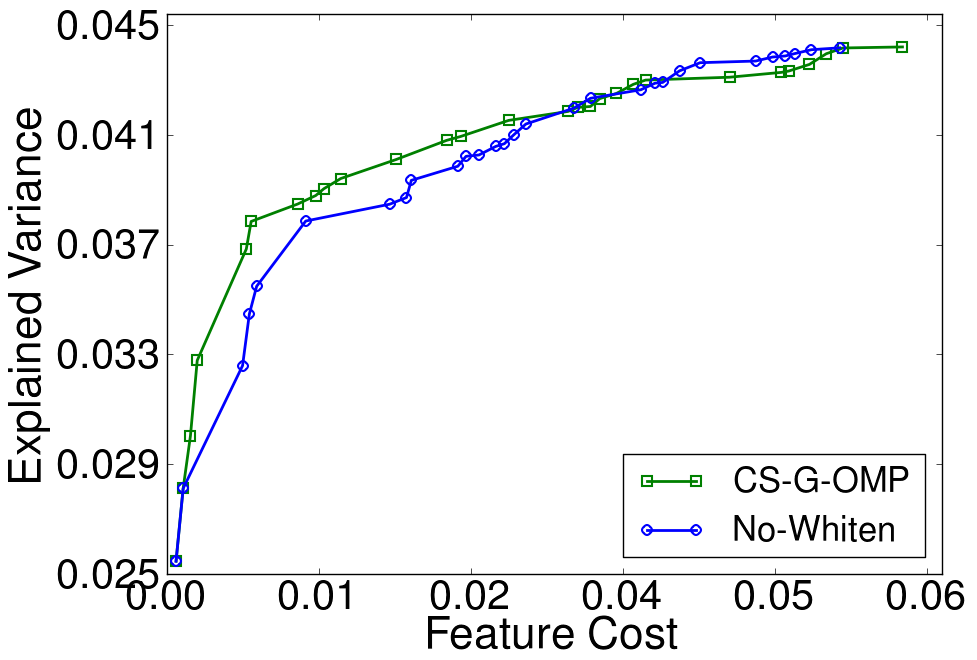}
}

\subfloat[Group Whiten vs. No-Whiten (\YahooLTR)]{
  \includegraphics[width=0.4\textwidth,height=3.8cm]{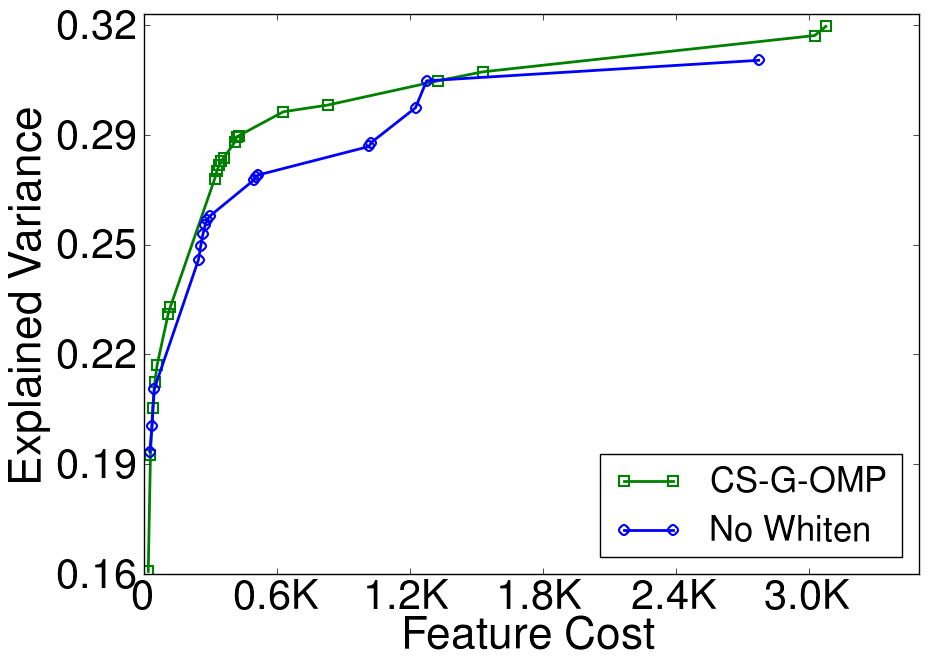}
}
\caption{Explained Variance vs. Feature Cost curves on \Grain\, (a) and \YahooLTR\, (b)  comparing group whitening with no group whitening. The curves stop at 0.97-stopping cost.}
\label{fig:whiten_vs_no_whiten}
\end{figure}

\subsection{SELECTION CRITERION VARIANTS}
\label{sec:selection_methods}

This section compares CS-G-OMP and CS-G-FR, along with 
variants of these two methods and the baseline, Sparse. 
We formulated the variant of CS-G-OMP, \textit{single}, in Section~\ref{sec:method} and it intuitively chooses feature groups of the best single feature dimension per group cost. Our experiments show that this modification degrades prediction performance of CS-G-OMP. 
Since FR directly optimizes the objective at each step, we expect CS-G-FR to perform the best and use its curve to compute the \textbf{Oracle} curve as an approximate to the best achievable performance.

In Figure~\ref{fig:selection_methods}, we evaluate CS-G-FR, CS-G-OMP and CS-G-OMP-single based on the objective in Theorem~\ref{thm:main}, i.e., explained variance vs. feature cost curves. 
CS-G-FR, as expected, outperforms all other methods. CS-G-OMP outperforms the baseline method, Sparse, and the CS-G-OMP-Single variant. 
The performance advantage of CS-G-OMP over CS-G-OMP-Single is much clearer in the \Grain\ data-set than in the \YahooLTR\ data-set. \Grain\ has a natural group structure which may contain correlated features in each group. \YahooLTR\ has a randomly generated group structure whose features were filtered by feature selection before the data-set was published \citep{yahoo_ltr}. CS-G-FR and CS-G-OMP outperform the baseline algorithm, Sparse. We speculate that linearly scaling group regularization constants by group costs did not enforce Group-Lasso to choose the most cost-efficient features early. 
The test-time timeliness measures of each of the methods are recorded in Table~\ref{tab:grain_auc} and Table~\ref{tab:yahoo_auc},
and quantitatively confirm the analysis above. Since \Grain\, and \YahooLTR\, are originally a classification and a ranking data-set, respectively, we also report in Figure~\ref{fig:selection_methods} the performance using classification accuracy and NDCG@5. This demonstrates the same qualitatively results as using explained variants. 


\begin{figure}
\centering
\subfloat[FR vs. OMP vs. Sparse (\Grain)]{
  \includegraphics[width=0.40\textwidth,height=4cm]{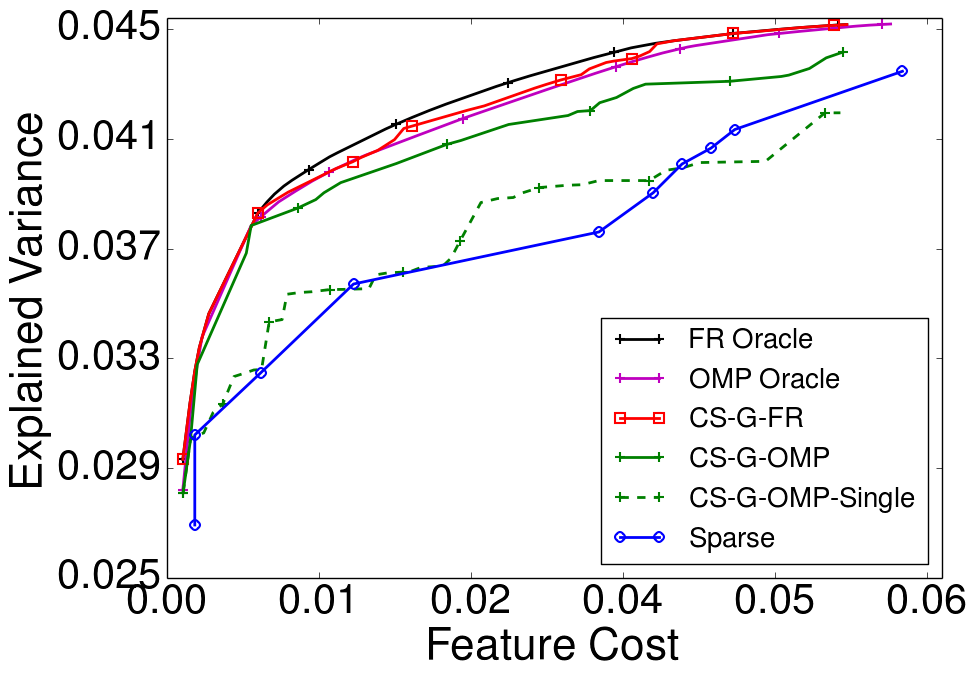}
}

\subfloat[FR vs. OMP vs. Sparse (\YahooLTR)]{
  \includegraphics[width=0.40\textwidth,height=4cm]{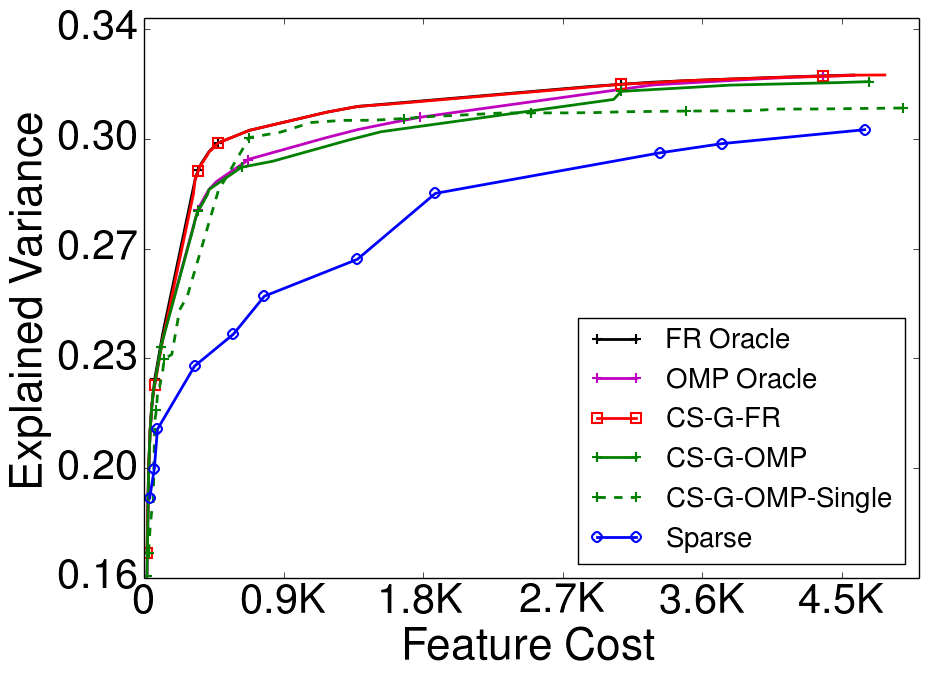}
}

\subfloat[FR vs. OMP vs. Sparse (\Grain)]{
  \includegraphics[width=0.40\textwidth,height=4cm]{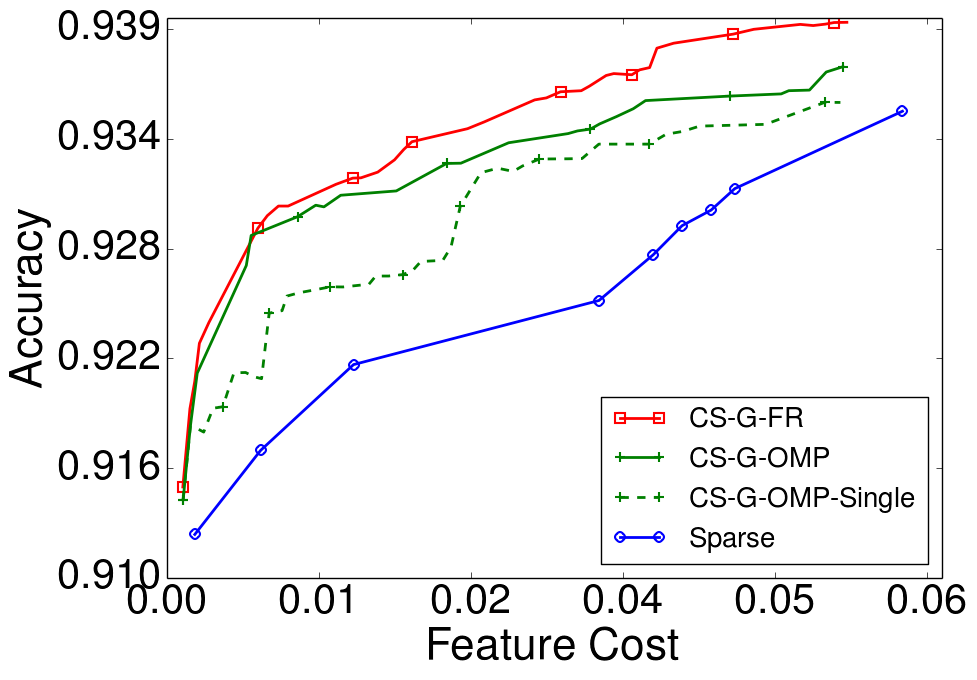}
}

\subfloat[FR vs. OMP vs. Sparse (\YahooLTR)]{
  \includegraphics[width=0.40\textwidth,height=4cm]{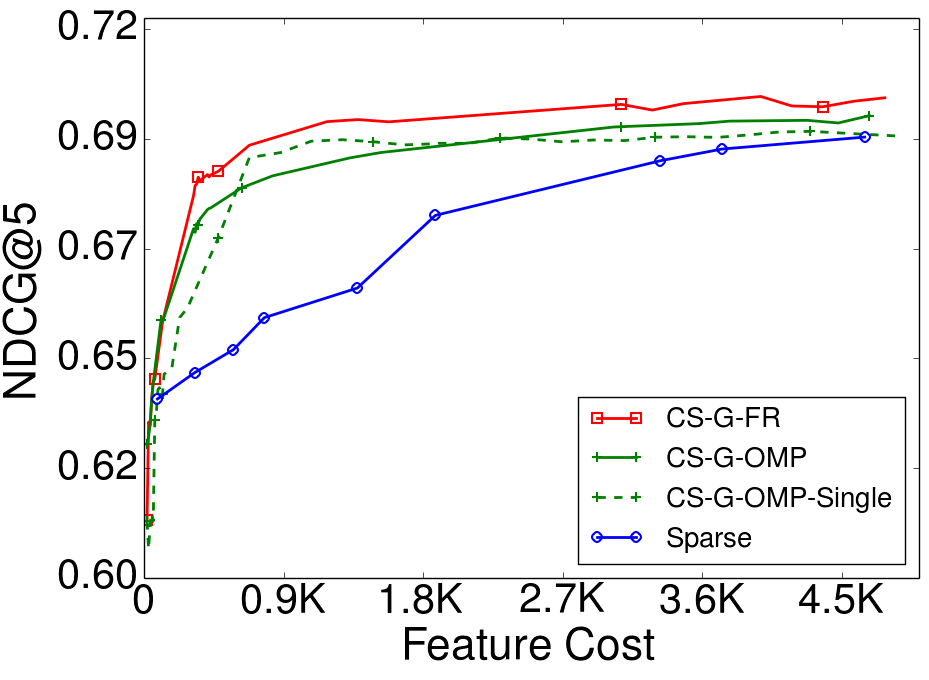}
}

\caption{(a),(b): Explained Variance vs. Feature Cost curves on 
\Grain\, and \YahooLTR (group-size=10), 
using CS-G-OMP, CS-G-FR and their Single variants. Curves stop at 0.97 and 0.98 stopping costs. (c),(d): Same curve with the natural objectives of the data-sets: accuracy and NDCG@5.} 
\label{fig:selection_methods}
\end{figure}

As expected, when compared against CS-G-OMP, CS-G-FR consistently chooses more cost-efficient features at the cost of a longer training time.
In the context of linear regression, let us assume that the group sizes are 
bounded by a constant when we are to select the number 
$K$ feature group. We can then compute a new model of $K$  groups in $O(K^2N)$ using
Woodbury's matrix inversion lemma, evaluate it in $O(KN)$, and compute the gradients with respect to the weights of unselected groups in $O(N(J-K))$. Thus, CS-G-OMP requires $O(K^2N + JN)$ at step $K=1,2,3,..., J$ and CS-G-FR requires $O((J-K)K^2N)$, so the total training  complexities for CS-G-OMP and CS-G-FR are $O(J^3N)$ and $O(J^4N)$, using $\sum_{K=1}^J K^2 = \frac{1}{6}J(J+1)(2J+1)$ and $\sum _{K=1}^J K^3 = \frac{1}{4}J^2(J+1)^2$. 
We also show this training complexity gap empirically in Figure~\ref{fig:run_time}, which plots the curves of training time vs. number of feature groups selected. When all feature groups are selected, CS-G-OMP achieves a 8x speed-up in \Grain\ over CS-G-FR. In \YahooLTR, CS-G-OMP achieves a speed-up factor between 10 and 20; the smaller the sizes of the groups, the larger speed-up due to the increase in the number of groups. Both greedy methods are much faster than the Lasso path computation using SPAMS, however.

\let\thefootnote\relax\footnote{
This work was conducted in part through collaborative participation in the Robotics Consortium sponsored by the U.S Army Research Laboratory under the Collaborative Technology Alliance Program, Cooperative Agreement W911NF-10-2-0016.}

\bibliographystyle{plainnat}
\bibliography{final}
\clearpage
\appendix

\section{Additional Proof Details}
\label{sec:proof_II}
This section describes a functional boosting view of selecting features for generalized linear models of one-dimensional response. We then prove Lemma~\ref{lemma:smoothness} and Lemma~\ref{lemma:convexity} for this more general setting. These more general
results in turn extend Theorem~\ref{thm:main} to generalized 
linear models.

\subsection{Functional Boosting View of Feature Selection}
\label{sec:functional}

We view each feature $f$ as a function 
$h_f$ that maps sample $x$ to $x_f$. We define $f_S: \R^{D} \rightarrow \R$ 
to be the best linear predictor using features in $S$, i.e., $f_S(x) \triangleq w(S)^Tx_S$. For each feature dimension $d \in D$, the coefficient of 
$d$ is in $w(S)$ is $w(S)_d = f_S(e_d)$, where $e_d$ is the $d^{th}$ dimensional unit vector. So $\Vert w(S) \Vert_2^2 = \sum _{d = 1}^D \Vert f_S(e_d) \Vert _2^2$. 
Given a generalized linear model with link function $\nabla \Phi$, 
the predictor is $E[ y | x ] = \nabla \Phi(w^Tx)$ for some $w$ and the calibrated loss is $r(w) = \sum _{i=1}^n (\Phi(w^Tx_i) - y_iw^Tx_i)$. 
Replacing $f_S(x_i) = w(S)^Tx_i$, we have 
\begin{align}
	r(w(S)) = \sum _{i=1}^n (\Phi(f_S(x_i)) - y_if_S(x_i)).
	\label{eq:glm_loss_general}
\end{align}
Note that the risk function in Equation~\ref{eq:risk} can be rewritten as 
the following to resemble Equation~\ref{eq:glm_loss_general}:
\begin{align}
 R(S) = 
 	\mathcal{R}[f_S] =& \frac{1}{n} \sum _{i=1}^n (\Phi(f_S(x_i)) - y_i^Tf_S(x_i)) \notag \\
    &{} + \frac{\lambda}{2} \sum _{d = 1}^D \Vert f_S(e_d) \Vert _2^2 + A,
  \label{eq:min_func}
\end{align}
where $\phi(x) = \frac{1}{2}x^2$ for linear predictions and constant 
$A = \frac{1}{2n} \sum _{i=1}^n y_i^2$. 
Next we define the inner product between two functions $f, h : \R^D \rightarrow \R$ 
over the training set to be:
\begin{align}
\angleb{f, h} \triangleq \frac{1}{n} 
  \sum _{i=1}^n f(x_i)h(x_i) + \frac{\lambda}{2} \sum _{d=1}^D f(e_d)h(e_d).
\end{align}
With this definition of inner product, we can compute the derivative of 
$\mathcal{R}$:
\begin{align}
  \nabla \mathcal{R}[f]  = \sum_{i=1}^n (\nabla\Phi(f(x_i)) 
  	- y_{i})\delta_{x_i} 
    + \sum _{d=1}^D f(e_d)\delta_{e_d},
\end{align}
where $\nabla \phi(x) = x$ for linear predictions, and 
$\delta_x$ is an indicator function for $x$. 
Then the gradient of objective $F(S)$ w.r.t coefficient $w_f$ of a feature dimension  $d$ can be written as:
\begin{align}
 b_{d}^S &= - \frac{1}{n}\sum_{i=1}^n (\nabla\Phi_p(w(S)^Tx^i) - y^{i})x^i_d - \lambda w(S)_d \\
    &= - \angleb{  \nabla \mathcal{R}[f_S], h_d }.
\end{align}
In addition, the regularized covariance matrix of features $C$ satisfies,
\begin{equation}
    C_{ij} = \frac{1}{n} X_i^TX_j + \lambda I(i=j) = \angleb{ h_i, h_j},
\end{equation}
for all $i, j = 1,2,..., D$.
So in this functional boosting view,  Algorithm~\ref{algo:gomp_lm} greedily chooses group $g$ that maximizes, with a slight abuse of notation of $\angleb{ \;, \; }$,
$\Vert \angleb{ h_g , \nabla \mathcal{R}[f_S] } \Vert _2^2 / c(g)$, i.e., 
the ratio between similarity of a feature group and the functional gradient, 
measured in sum of square of inner products, and the cost of the group

\subsection{Proof of Lemma~\ref{lemma:smoothness} and Lemma~\ref{lemma:convexity}}

The more general version of Lemma~\ref{lemma:smoothness} and Lemma~\ref{lemma:convexity} assumes that the objective functional $\mathcal{R}$ 
is $m$-strongly smooth and $M$-strongly convex using our proposed inner product rule.  
$M$-strong convexity is a reasonable assumption, 
because the regularization term $\Vert w \Vert _2^2 = \sum _{d = 1}^D \Vert f_S(e_d) \Vert _2^2$ ensures that all loss functional $\mathcal{R}$ with a convex $\Phi$ 
 strongly convex. 
In the linear prediction case, both $m$ and $M$ equals $1$. 

The following two lemmas are the more general versions of Lemma~\ref{lemma:smoothness} and Lemma~\ref{lemma:convexity}.
\begin{lemma}
  Let $\mathcal{R}$ be an m-strongly smooth functional with respect to our definition of inner products. Let $S$ and $G$ be some fixed sequences. Then
  \begin{align}
   F(S) - F(G) \notag \leq \frac{1}{2m} \angleb{b^G_{G \oplus S}, C_{G \oplus S}^{-1} b^G_{G\oplus S}}
  \end{align}
  \label{lemma:smoothness_glm}
\end{lemma}
\begin{proof}
First we optimize over the weights in $S$. 
  \begin{align*}
    &{} F(S) - F(G) \\
    &= \mathcal{R}[f_G] - \mathcal{R}[f_S] 
     = \mathcal{R}[f_G] - \mathcal{R}[\sum _{s \in S} \alpha_s^T h_s] \\
    &\leq \mathcal{R}[f_G] - \min _{w : w_i^T \in \R^{d_{s_i}}, s_i \in S} 
        \mathcal{R}[ \sum _{s_i \in S} w_{s_i}^T h_{s_i}] \\
\intertext{Adding dimensions in $G$ will not increase the risk, we have: }
    &\leq \mathcal{R}[f_G] - \min _{w : w_i \in \R^{d_{s_i}}, s_i \in G \oplus S}
        \mathcal{R}[ \sum _{s_i \in G \oplus S} w_{s_i} h_{s_i}] \\
\intertext{Since $f_G = \sum _{g_i \in G} \alpha_i h_{g_i}$, we have:}
    &\leq \mathcal{R}[f_G] - \min _{w} 
      \mathcal{R}[f_G + \sum _{s_i \in G \oplus S} w_i^T h_{s_i}] \\
\intertext{Expanding using strong smoothness around $f_G$, we have:}
    &\leq \mathcal{R}[f_G] - \min _{w} (
      \mathcal{R}[f_G] + \angleb{ \nabla \mathcal{R}[f_G], 
        \sum _{s_i \in G\oplus S} w_i^T h_{s_i} } \notag \\
    &\quad + \frac{m}{2} \Vert \sum _{s_i \in G \oplus S} w_i^T h_{s_i} \Vert _2^2) \\
    &= \max_{w} - 
      \angleb{ \nabla \mathcal{R}[f_G], 
      \sum _{s_i \in G\oplus S} w_i^T h_{s_i} } 
        - \frac{m}{2} \Vert \sum _{s_i \in G \oplus S} w_i^T h_{s_i} \Vert _2^2 \\
    &= \max_w \angleb{ b^{G}_{G\oplus S}, w} - \frac{m}{2} \angleb{w, C_{G\oplus S}w}
 \end{align*}
Solving $w$ directly we have:
\begin{align*}
  F(S) - F(G) \leq \frac{1}{2m} \angleb{ b^{G}_{G\oplus S} , C_{G\oplus S}^{-1} b^{G}_{G\oplus S}}
\end{align*}
\end{proof}

\begin{lemma}
  Let $\mathcal{R}$ be a M-strongly convex functional with respect to our definition of 
  inner products. Then
    \begin{align}
      F(G_j) - F(G_{j-1}) \geq \frac{1}{2M (1 + \lambda) } \angleb{{b^{G_{j-1}}_{g_j}}, b^{G_{j-1}}_{g_j} }
    \end{align}
  \label{lemma:convexity_glm}
\end{lemma}
\begin{proof}
After the greedy algorithm chooses some group $g_j$ at step $j$, 
  we form $f_{G_j} = \sum _{\alpha _i} \alpha_i^T h_{g_i}$, such that
   \[
    \mathcal{R}[f_G] = \min _{ \alpha _i \in \R^{d_{g_i}}} 
      \mathcal{R}[ \sum _{g_i \in G_j} \alpha_i^T h_{g_i}] \leq
      \min _{\beta \in \R^{d_{g_j}}} 
    \mathcal{R}[f_{G_{j-1}} + \beta h_{g_j}]
   \]
 Setting $\beta = \arg \min _{\beta \in \R^{d_{g_j}}} 
    \mathcal{R}[f_{G_{j-1}} + \beta h_{g_j}]$, using the strongly convex condition at
      $f_{G_{j-1}}$, we have:
 \begin{align*}
    &{} F(G_j) - F(G_{j-1}) \\
    &=  \mathcal{R}[f_{G_{j-1}}] - \mathcal{R}[f_{G_j}] 
    \geq \mathcal{R}[f_{G_{j-1}}] - \mathcal{R}[f_{G_{j-1}} + \beta h_{g_j}] \\ 
    &\geq  \mathcal{R}[f_{G_{j-1}}] - 
      (\mathcal{R}[f_{G_{j-1}}] + 
        \angleb{ \nabla \mathcal{R}[f_{G_{j-1}}] , 
        \beta h_{g_j} } \notag \\
    &\quad + \frac{M}{2} \Vert \beta h_{g_j} \Vert _2^2) \\
    &=  -  \angleb{ \nabla \mathcal{R}[f_{G_{j-1}}] , 
        \beta h_{g_j} }
       - \frac{M}{2} \Vert \beta h_{g_j} \Vert _2^2 \\
    &=  \angleb{{b^{{G_{j-1}}}_{g_j}},  \beta} - \frac{M}{2} \angleb{\beta, C_{g_j} \beta} \\
    &\geq  \frac{1}{2M} \angleb{ b^{{G_{j-1}}}_{g_j}, C_{g_j}^{-1} b^{{G_{j-1}}}_{g_j}} \\
    &=  \frac{1}{2M (1+\lambda)} \angleb{{b^{{G_{j-1}}}_{g_j}}, b^{{G_{j-1}}}_{g_j}}
 \end{align*}
 The last equality holds because each group is whitened, 
 so that $C_{g_j} = (1 + \lambda) I$.
\end{proof}
Note that the $(1+\lambda)$ constant is a result of group whitening, without which
the constant can be as large as $(D_{g_j} + \lambda)$ for  the worst case where
all the $D_{g_j}$ number of features are the same. \\

The proofs above for Lemma~\ref{lemma:smoothness_glm} and~\ref{lemma:convexity_glm} are 
for one-dimensional output responses. They can be easily generalized to multi-dimensional 
responses by replacing 2-norms with Frobenius norms and vector inner-products with ``Frobenius products", i.e., the sum of the products of all elements. \\

\subsection{Proof of Main Theorem}
\label{sec:app-main-proof}
Given Lemma~\ref{lemma:smoothness_glm} and Lemma~\ref{lemma:convexity_glm}, 
the proof of Lemma~\ref{lemma:main} holds with the same analysis with a more 
general constant $\gamma = \frac{m \lambda_{min}(C)}{M(1+\lambda)}$. The following prove our main theorem~\ref{thm:main}. 
 
\begin{proof} (of Theorem~\ref{thm:main}, given Lemma~\ref{lemma:main})
\label{proof:main}
  Define $\Delta _j = F(S_{\angleb{K}}) - F(G_{j-1})$. Then we have 
  $\Delta _j - \Delta_{j+1} = F(G_{j}) - F(G_{j-1})$. By Lemma ~\ref{lemma:main}, we have:
  \begin{align*}
    \Delta_j &= F(S_{\angleb{K}}) - F(G_{j-1})\\
    &\leq \frac{K}{\gamma}
      \lbrack \frac{F(G_{j}) - F(G_{j-1})}{c(g_{j})} \rbrack 
        = \frac{K}{\gamma} \lbrack \frac{\Delta_j - \Delta_{j+1}}{c(g_j)} \rbrack
  \end{align*}
  Rearranging we get 
    $\Delta_{j+1} \leq \Delta_j ( 1 - \frac{\gamma c(g_j)}{K} )$. Unroll we get:
  \begin{align*}
    \Delta _{L+1} 
    &\leq 
      \Delta _1 \prod _{j=1}^L (1 - \frac{\gamma c(g_j)}{K})
    \leq \Delta _1 ( \frac{1}{L} \sum _{j=1}^L (1- \frac{\gamma c(g_j)}{K})) ^L\\
    &= \Delta _1 (1 - \frac{B\gamma}{L K})^L < \Delta_1 e^{- \gamma \frac{B}{K}}
  \end{align*}
  By definition of $\Delta_1$ and $\Delta_{L+1}$, we have:
  \begin{align*}
    F(S_{\angleb{K}}) - F(G_{\angleb{B}}) < F(S_{\angleb{K}}) e^{- \gamma \frac{B}{K}}
  \end{align*}
  The theorem follows and linear prediction is the special case that $m = M$.
\end{proof}

\section{Extension to Generalized Linear Model}
\label{sec:extension}
While we only formulated the feature group sequencing problem in linear prediction setting previously, we can extend our algorithm for generalized linear models and multi-dimensional responses. In general, we assume that we have $P$ dimensional responses, and predictions are of the form $E[y | x ] = \nabla \phi( W x)$, for some known convex function $\phi : \mathbb{R}^P \rightarrow \mathbb{R}$, and an unknown coefficient $P\times D$ matrix, $W$. Thus, the generalized linear
prediction problem is to minimize over coefficient matrix $W: P\times D$:
\begin{align}
	\textbf{r}(W) = \frac{1}{n} \sum _{i=1}^n (\phi(Wx^i) - y_i^TWx_i) 
		+ \frac{\lambda}{2} \Vert W\Vert^2_F,
		\label{eq:glm_loss}
\end{align}
where $\lambda$ is the regularization constant for Frobenius norm of the coefficient 
matrix. In particular, we have $\phi(x) = \frac{1}{2}x^2$ for linear prediction. 
The risk of a collection of features, $S$, is then 
\mbox{$R(S) = \underset{W : \forall g \notin S  W_g = \textbf{0} }{\min} \textbf{r}(W)$}. 
To extend CS-G-OMP to feature sequencing in this general setting, we again, at each step,
take gradient of the objective $\textbf{r}$ w.r.t. $W$, and choose the feature group that has the largest ratio of group gradient Frobenius norm square to group cost. More specifically, after choosing groups in $G$, we have a best coefficient matrix restricted to G, $W(G)$. Then we compute the gradient w.r.t. $W$ at $W(G)$ (we keep the convention that unselected groups have zero coefficients) as:
\begin{align}
	\nabla \textbf{r}(W) = \frac{1}{n} \sum _{i=1}^n (\nabla \phi(Wx^i) - y_i)x_i^T + \lambda W;
	\label{eq:glm_gradient}
\end{align}
we then evaluate $\Vert \textbf{r}(W)_g \Vert_F^2 / c(g)$ for each feature group $g$, and add the maximizer to the selected groups to create new models. 
Algorithm~\ref{algo:gomp_glm} demonstrates the procedure. 

Our theoretical result Theorem~\ref{thm:main} can also be proven in this general setting. 
Proofs of Lemma \ref{lemma:smoothness} and \ref{lemma:convexity}) in appendix are readily for generalized linear models\footnote{Inner products, $\angleb{\bullet, \bullet}$, in Lemma \ref{lemma:smoothness} and \ref{lemma:convexity} now represent Frobenius products, which are sums of element-wise products of matrices.}. Given these two lemmas, our proofs of Lemma~\ref{lemma:main}
and Theorem~\ref{thm:main} hold as they are.

\IncMargin{1em}
\begin{algorithm}[h]
  \SetKwInOut{Input}{input}\SetKwInOut{Output}{output}

  \Input{The data matrix $\XX = [ \textbf{f}_1, ..., \textbf{f}_D ] \in \R^{n \times D}$,
    with group structures, such that for each group $g$,
    $\XX^T_{g}\XX_{g} = I _{D_g}$.     
    The cost $c(g)$ of each group $g$. 
    The response matrix $\YY \in \{0, 1\}^{n \times P}$.
    The link function $\nabla \Phi$.
    Regularization constant $\lambda$.
  }
  \Output{
    A sequence $((G_j, W_j))_j$, where 
      $G_j = (g_1, g_2, ..., g_j)$ 
    is the sequence of first $j$ selected feature groups, $g_1, g_2, ..., g_j$, and 
      $W_j: P\times D$ restricted to features in $G_j$ 
      is the associated coefficient matrix.
  }

  $G_0 = \emptyset$; $W_0 = \textbf{0}$\;
  \For{$j = 1, 2, ...$}{
    Compute $\textbf{r}' = \nabla \textbf{r}(W_{j-1})$ with Eq.~\ref{eq:glm_gradient}\;
    \tcp{Selection step (*)} 
    $g_j = \arg \max _g \Vert \textbf{r}'_g \Vert_F^2 / c(g)$\;
    \tcp{Append selected group}
    $G_j = G_{j-1} \oplus (g_j)$\;
    \tcp{Solve for the best model with selected feature}
    Use a GLM algorithm to minimize Eq.~\ref{eq:glm_loss} restricted to 
    features in $G_j$\\ 
    \Indp $W_{j} = \underset{W : \forall g \notin G_j W_g = \textbf{0} }{\arg \min} R(W) $\;
  }
  \caption{Cost Sensitive Group Orthogonal Matching Pursuit (G-OMP) for Generalized Linear Models}
  \label{algo:gomp_glm}
\end{algorithm}

\end{document}